\documentclass[preprint,review,12pt]{elsarticle}

\usepackage{amssymb}
\usepackage{amsmath}

\usepackage{algorithm}
\usepackage{microtype}
\usepackage{graphicx}
\usepackage{subfigure}
\usepackage{tabularx}
\usepackage{multicol}
\usepackage{booktabs} 
\usepackage{mathtools}
\usepackage{amsthm}
\usepackage{enumitem}
\usepackage{diagbox}

\usepackage[capitalize,noabbrev]{cleveref}

\theoremstyle{plain}
\newtheorem{theorem}{Theorem}[section]

\newtheorem{lemma}[theorem]{Lemma}

\theoremstyle{definition}

\theoremstyle{remark}

\DeclareMathOperator*{\argmax}{arg\,max}

\usepackage{algpseudocode}
\usepackage{multirow}
\usepackage{array}
\usepackage{amsmath}
\usepackage{xcolor}
\usepackage{mathtools}
\usepackage{amsthm}
\usepackage{lipsum,booktabs}
\usepackage{amsfonts}
\newcolumntype{P}[1]{>{\centering\arraybackslash}p{#1}}
\newcolumntype{L}{>{\centering\arraybackslash}m{3cm}}
\usepackage{newfloat}
\usepackage{listings}
\usepackage{bm}
\usepackage{microtype}
\usepackage{caption}
\usepackage{url}

\journal{Information Sciences}

\begin{document}

\begin{frontmatter}



\title{Tabular Diffusion Counterfactual Explanations} 


\author{Wei Zhang} 

\affiliation{organization={Electrical Engineering, Columbia University},
            addressline={116th and Broadway}, 
            city={New York},
            postcode={10025}, 
            state={NY},
            country={U.S.A.}}

\author{Brian Barr} 

\affiliation{organization={Capital One},
            addressline={1680 Capital One Dr}, 
            city={Mc Lean},
            postcode={22102}, 
            state={VA},
            country={U.S.A.}}

\author{John Paisley} 

\affiliation{organization={Electrical Engineering, Columbia University},
            addressline={116th and Broadway}, 
            city={New York},
            postcode={10025}, 
            state={NY},
            country={U.S.A.}}

\begin{abstract}
Counterfactual explanations methods provide an important tool in the field of {interpretable machine learning}. Recent advances in this direction have focused on diffusion models to explain a deep classifier. However, these techniques have predominantly focused on problems in computer vision. In this paper, we focus on tabular data typical in finance and the social sciences and propose a novel guided reverse process for categorical features based on an approximation to the Gumbel-softmax distribution. Furthermore, we study the effect of the temperature $\tau$ and derive a theoretical bound between the Gumbel-softmax distribution and our proposed approximated distribution. We perform experiments on several large-scale credit lending and other tabular datasets, assessing their performance in terms of the quantitative measures of interpretability, diversity, instability, and validity. These results indicate that our approach outperforms popular baseline methods, producing robust and realistic counterfactual explanations.
\end{abstract}


\begin{highlights}
\item We propose a novel counterfactual generation algorithm for tabular datasets using Gumbel-softmax re-parameterization in controllable diffusion models. Our method permits gradient backpropagation, which resembles the classifier guidance in the Gaussian case.  
\item We introduce an approximation to the Gumbel-softmax distribution and derive a tight bound. We also study the effect of temperature $\tau$ in the Gumbel-softmax distribution. 
\item We experiment with our method on four large-scale tabular datasets. The results demonstrate that our method can achieve competitive performance using widely-adopted metrics for counterfactual generation. 
\end{highlights}

\begin{keyword}
Counterfactual Generation \sep Controllable Diffusion Models \sep Explainable Machine Learning



\end{keyword}

\end{frontmatter}



\section{Introduction}
Deep neural networks have revolutionized many fields, perhaps most notably in computer vision and natural language processing. Despite its extraordinary performance, the frequent lack of a deep model's explainability prevents it from being widely adopted in regulated fields such as Fintech. Practitioners in those fields are interested in not only the decisions given by the black-box model, but also the reasons behind the decisions. This is necessary for transparency of the factors impacting the decision, and explaining alternatives that may produce different outcomes. 

Many methods have been developed to improve transparency of back-box models. A number of works generate feature importance based on local approximations \cite{ribeiro2016should}, global approximations \cite{ibrahim2019global}, gradients attributions \cite{shrikumar2017learning,sundararajan2017axiomatic} and SHAP values \cite{lundberg2017unified}. Other methods focus on interventions in causal regimes \cite{liu2024actionability, liu2018delayed} and the construction of additive models using neural networks \cite{agarwal2021neural,radenovic2022neural,chang2021node, zhang2024gaussian}.

In this paper, we focus on tabular counterfactual explanations (CEs), a post-hoc method that answers the question ``What changes can be made to an input so that its output label changes to the target class?'' CEs aim to explain a classifier $f:\mathbb{R}^d\rightarrow \lbrace 0,1\rbrace$ by generating a counterfactual sample $\hat{x}$ such that the predicted label is flipped with minimal changes to the input ad defined by a metric $d(\cdot,\cdot)$. This can be characteristically formulated as
\begin{equation}\label{eqn:cf_obj}
    \arg\min_{\widehat{x}} d(x, \hat{x}) \quad \text{subject to} \quad f(\widehat{x}) = y_{\mathrm{target}}.
\end{equation}
\citet{wachter2017counterfactual} cast this framework into an optimization problem and directly back-propagate the gradients of the classifier and distance constraints into the feature space. This approach treats each feature as a continuous variable and thus does not directly apply to categorical features. Other methods explicitly deal with categorical features and generate the counterfactual explanations using graphs \cite{poyiadzi2020face}, prototypes \cite{van2021interpretable}, multi-objective functions \cite{dandl2020multi}, rule-based sets \cite{guidotti2018local}, point processes \cite{mothilal2020explaining} and random forests \cite{fernandez2020random}.

Deep generative models such as VAEs \cite{kingma2013auto} also play a key role because of their high-fidelity generative power. \citet{joshi2019towards} and \citet{antoran2020getting} propose methods that search in the latent space of a VAE to generate counterfactuals. \citet{pawelczyk2020learning} focus explicitly on tabular data and use conditional VAEs as the generator for a target class. Methods in this line of work build on VAE architectures and rely on an efficient searching algorithm. 

Instead, like \citet{wachter2017counterfactual} we work directly in the feature space, but approach the problem from the perspective of diffusion modeling \cite{ho2020denoising, song2020denoising}. In the continuous image domain \citet{dhariwal2021diffusion} have introduced classifier guidance on the reverse process for continuous features such as image data, while \citet{augustin2022diffusion} built upon this framework with a counterfactual constraint on the reverse process. The proposed methods generate high-fidelity counterfactual images for a vision classifier. Explainable diffusion models for categorical tabular data have received less consideration. While diffusion models have been extensively studied for categorical tabular data, e.g., \cite{hoogeboom2021argmax, sun2022score,dieleman2022continuous,kotelnikov2023tabddpm,regol2023diffusing}, this line of work rarely intends to provide explanations for a classifier. 

Two notable recent investigations in this area include \cite{gruver2024protein,schiff2024simple}. \citet{gruver2024protein} focus on controllable discrete diffusion models in protein design by introducing a learnable mapping function that projects a discrete vector onto a continuous representation. The resulting representation is treated as a continuous variable and is diffused through the Gaussian distribution. \citet{schiff2024simple} also focuses on discrete data by treating a one-hot vector as continuous.

We propose a novel tabular diffusion model for counterfactual explanations that leverages Gumbel-softmax re-parameterization \cite{jang2016categorical}. Our contributions are three-fold: 
\begin{enumerate}[itemsep=0pt,leftmargin=*]
    \item Our method permits gradient backpropagation, and the resulting reverse process resembles the classifier guidance in the Gaussian case. It is easy to implement and efficient for counterfactual generation.
    \item We study the effect of temperature $\tau$ in the Gumbel-softmax distribution on our model and derive a tight bound between an introduced approximation. Our proposed method approximates the base model better as the temperature decreases.
    \item We experiment on four large-scale tabular datasets. The results demonstrate that our method achieves competitive performance on popular metrics used to evaluate counterfactuals within the field.
\end{enumerate}

\section{Related Works}
To situate our method within technological developments, we first review machine learning works related to counterfactual explanations and recent advances in controllable diffusion models, highlighting some of the key differences and shortcomings our method seeks to address for tabular data. 

\subsection{Counterfactual Explanations}
Following Equation \ref{eqn:cf_obj}, researchers have leveraged the auto-encoder architecture to construct a counterfactual explainer. \citet{joshi2019towards} takes a learned auto-encoder and aims to find the latent vector of the counterfactual sample by back-propagating gradients from the classifier into the latent space. The distance constraint is applied in the feature space. \citet{antoran2020getting} takes a similar approach but uses Bayesian Neural Networks to estimate the uncertainty of the generated counterfactual. \citet{pawelczyk2020learning} also works in the latent space but explicitly tackles a set of immutable features. The authors use the conditional HVAE \cite{nazabal2020handling} and condition on the immutable features while the searching of the counterfactual is again completed in the latent space with validity and minimum changes constraints in the feature space. These methods work in the latent space, and once the latent vector is found, the counterfactual sample is generated from the pretrained decoder.

To mitigate the searching task, \citet{guo2023counternet,guo2023rocoursenet, zhang2022interpretable} train the classifier and counterfactual generator simultaneously by supervising the latent space. The counterfactual samples can be generated by linear mapping\cite{zhang2022interpretable} and non-linear mapping \cite{guo2023counternet, guo2023rocoursenet} in the latent space, which effectively reduces the computational cost. Though efficient, such explainers are model-dependent, and the uncertainty of the decoder still exists. In contrast, our approach will be model-agnostic and directly operates in the feature space. 

On the other hand, \citet{wachter2017counterfactual} directly works in the feature space. The proposed method back-propagates the gradients that lead to the target class label with minimum changes in the feature space. Through this back-propagation, it becomes easier to handle immutable features, which simply mask the corresponding gradients. Though effective, it is still hard to handle categorical features in this setting. In addition, it has been shown that a single pixel can fool a well-trained classifier \cite{su2019one}. Thus, although the resulting counterfactual sample might be valid (i.e. changed its label), it may not provide meaningful human-actionable information, or insight into the learned deep neural network. 

Another method called FACE \cite{poyiadzi2020face} also works in the feature space. Here, a graph is first constructed based on the existing dataset. A graph search algorithm is then performed until it finds the counterfactual sample with the target label and minimum changes. If immutable features are present, a sub-graph is selected from the original graph. Though intuitive, the counterfactual samples are only selected from the existing dataset, which limits the diversity of the generated samples. Depending on the size of the dataset, the proposed method might also suffer from an instability issue. The computational cost is also high for this method.

\subsection{Guided Diffusion Models}
Diffusion models have demonstrated much generative power for image generation \cite{song2020denoising,ho2020denoising} and tabular data generation \cite{kotelnikov2023tabddpm}. However, counterfactual explanations through diffusion models are still rapidly developing. Guided diffusion models have been extensively studied \cite{dhariwal2021diffusion} and extended \cite{augustin2022diffusion} in this direction. In this paper, we employ these developments for tabular data, but challenges remain for extending to categorical features. Guided diffusion works because they operate in continuous spaces, where gradients can be calculated. However, this is infeasible in discrete spaces.  

Recently, \citet{gruver2024protein} have developed a controllable diffusion pipeline for protein generation, which is purely categorical data, using a continuous function mapping.  \citet{schiff2024simple} also worked on categorical language data, by directly treating the categorical vector as if it were a continuous vector. Both of these recent works have demonstrated their efficacy for their related tasks. While in this line of work, our paper is distinct in two ways: 1) We propose a new approach to handling categorical data in controllable diffusion models that requires minimal modification to the existing tabular diffusion frameworks, and 2) We handle both continuous and categorical data simultaneously with the aim of explainable classification, whereas these two papers do not involve a classification problem. 

\section{Background on Diffusion Models}
\subsection{Tabular diffusion models}
Diffusion models have been extensively studied recently as a powerful generative model for high fidelity images  \cite{sohl2015deep,ho2020denoising,nichol2021improved, song2020denoising}. A typical diffusion model consists of a forward and reverse Markov process. The forward process injects Gaussian noise to the input along a sequence of time step, terminating at a prior, typically isotropic Gaussian distribution. The Markovian assumption factorizes the forward process as $q(x_{1:T}|x_0) = \prod_{t=1}^T q(x_t | x_{t-1})$. The reverse process aims to gradually denoise from the prior $x_T\sim q(x_T)$ to generate a new sample through $p(x_{0:T}) = \prod_{t=1}^T p(x_{t-1}|x_t)$. Although the Gaussian forward process can be derived in closed form, the reverse process $p(x_{t-1}|x_t)$ is intractable and requires a neural network to approximate. The parameters of the denoising neural network can be learned by maximizing the evidence lower bound,
\begin{align}\label{eq.elbo}
   \log q(x_0) &\geq \mathbb{E}_{q(x_0)}\Big[{\underbrace{\log q(x_0|x_1)}_{\mathcal{L}_0}} - \underbrace{\mathrm{KL}(q(x_T|x_0)\|q(x_T))}_{\mathcal{L}_T} \nonumber\\
   &- \sum\limits_{t=2}^T \underbrace{\mathrm{KL}(q(x_{t-1}|x_t, x_0)\|q(x_{t-1}|x_t))}_{\mathcal{L}_t}\Big].
\end{align}
The key distinction of tabular diffusion models is that there are two independent processes: Gaussian diffusion models for continuous features and Multinomial diffusion models for categorical features \cite{kotelnikov2023tabddpm}.

\paragraph{Continuous diffusions.} Let $x_t\in \mathbb{R}^{D}$ and $\alpha_t = 1 - \beta_t$ where $t\in [1,T]$ is the time step. The forward process follows the distribution
\begin{equation}
    q(x_{t}) \sim \mathcal{N} (x_{t} | \sqrt{\alpha_t} x_{t-1},(1 - \beta_t I)).
\end{equation}
Given $x_0$, the marginal distribution of $x_t$ for any $t$ is $q(x_t | x_0) \sim \mathcal{N}(x_t | \sqrt{\overline{\alpha}_t}x_0,\sqrt{1 - \overline{\alpha}_t} I)$ where $\overline{\alpha}_t = \prod_{i=1}^t\alpha_i$. This allows direct generation of the noisy $x_t$. The reverse process approximates the true posterior $q(x_{t-1} | x_t, x_0)$ with $q_{\theta}(x_{t-1} | x_t)$. By Bayes' rule, $q(x_{t-1} | x_t, x_0)$ can be computed in closed form and is Gaussian. Therefore, $q_{\theta}(x_{t-1} | x_t)$ is usually chosen to be a neural network-parameterized Gaussian, $q_{\theta}(x_{t-1} | x_t) \sim \mathcal{N}(x_{t-1} | \mu_{\theta}(x_t, t), \Sigma_{\theta}(x_t, t))$. Alternatively, \citet{ho2020denoising} found that, instead of directly producing the mean of the posterior Gaussian distribution, more favorable results can be found by predicting the noise at each time step:
\begin{equation}
    \mathcal{L}_t = \mathbb{E}_{\epsilon\sim\mathcal{N}(0,I)} \|\epsilon_t - \epsilon_{\theta}(x_t, t))\|^2
\end{equation}
where $\epsilon_{\theta}$ is a neural network. Once trained, the mean of the posterior can be derived as
\begin{equation}
    \mu_{\theta}(x_t, t) = \frac{1}{\sqrt{1-\beta_t}}\left(x_t - \frac{\beta_t}{\sqrt{1 - \overline{\alpha}_t}}\epsilon_{\theta}(x_t, t)\right),
\end{equation}
which gradually denoises $x_t$. Furthermore, \citet{ho2020denoising} construct the generative process using stochastic Langevin dynamics which introduce randomness during the sampling
process. We use the same dynamics, except for the final step which produces actual samples.

\paragraph{Categorical diffusions.} Multinomial diffusion models adapt the framework to handle categorical data \cite{hoogeboom2021argmax}. Let $x_t$ be a $K$-dimensional one-hot vector. The forward process now becomes 
\begin{equation}
    q(x_t | x_{t-1}) \sim \mathrm{Cat}\big(x_t|(1-\beta_t)x_{t-1} + \beta_t/K\big).
\end{equation}
When $T$ is large enough, the resulting $x_T \sim \mathrm{Cat}(x_T|1/K)$. Similar to the continuous case, $x_t$ can be computed as $q(x_t|x_0) = \mathrm{Cat}(x_t|\overline{\alpha}_t x_0 + (1-\overline{\alpha}_t)/K)$. The posterior of the reverse process can be derived using Bayes' rule,
\begin{equation}
    q(x_{t-1}|x_t,x_0) \sim \mathrm{Cat}\Big(x_{t-1}|\pi/\textstyle\sum_{i=1}^K \pi_i\Big),
\end{equation}
where $\pi = [\alpha_t x_t + (1-\alpha_t)/K]\odot [\overline{\alpha}_{t-1}x_0 + (1-\overline{\alpha}_{t-1})/K]$. The loss $\mathcal{L}_t$ in the categorical case is the KL divergence $\mathrm{KL}(q(x_{t-1}|x_t,x_0) \| p_{\theta}(x_{t-1}|x_t))$ where the neural network outputs the predicted $\widetilde{x}_0$ directly from the noisy input $x_t$.

\subsection{Classifier guidance}
Controllable reverse processes have been explored to generate class-dependent samples \cite{nichol2021improved}. In classifier-free guidance, the target class label $y$ is embedded into the denoising neural network, generating class-dependent predicted noise. No classifier exists to be explained or generate counterfactuals for, and so these techniques are outside the scope of this paper.

In classifier guidance methods, a differentiable classifier $p_{\phi}(y|x)$ is trained on the input space and a guided reverse process is formulated as
\begin{equation}\label{eqn:rev_joint}
    p_{\theta,\phi}(x_t|x_{t+1},y) = \textstyle\frac{1}{Z}p_{\theta}(x_t|x_{t+1})p_{\phi}(y|f_{dn}(x_t)),
\end{equation}
where $f_{dn}$ reconstructs the noise-free sample. A first-order Taylor expansion around the mean $\mu$ gives the approximation
\begin{equation}\label{eqn:cls_guid_rev}
    \textstyle\frac{1}{Z}p_{\theta}(x_t|x_{t+1})p_{\phi}(y|x_t) \sim \mathcal{N}(\mu + \Sigma g, \Sigma),
\end{equation}
where $g = \nabla_{x_t}\log p_{\phi}(y | f_{dn}(x_t))|_{x_t=\mu}$. We see that the reverse process uses gradient information from the target class in the generative process. 

However, in the categorical setting a combinatorial challenge arises when calculating gradients, resulting in $\mathcal{O}(\prod_i K_i)$ forward passes from the classifier, where $K_i$ is the number of options for $i$-th categorical variable. This is infeasible when the number of categorical variables becomes large. This challenge motivates our following use of Gumbel-softmax reparameterization, resulting in a reverse process similar to that of the continuous case. 

\begin{figure}[t]
    \centering
    \includegraphics[width=0.5\linewidth]{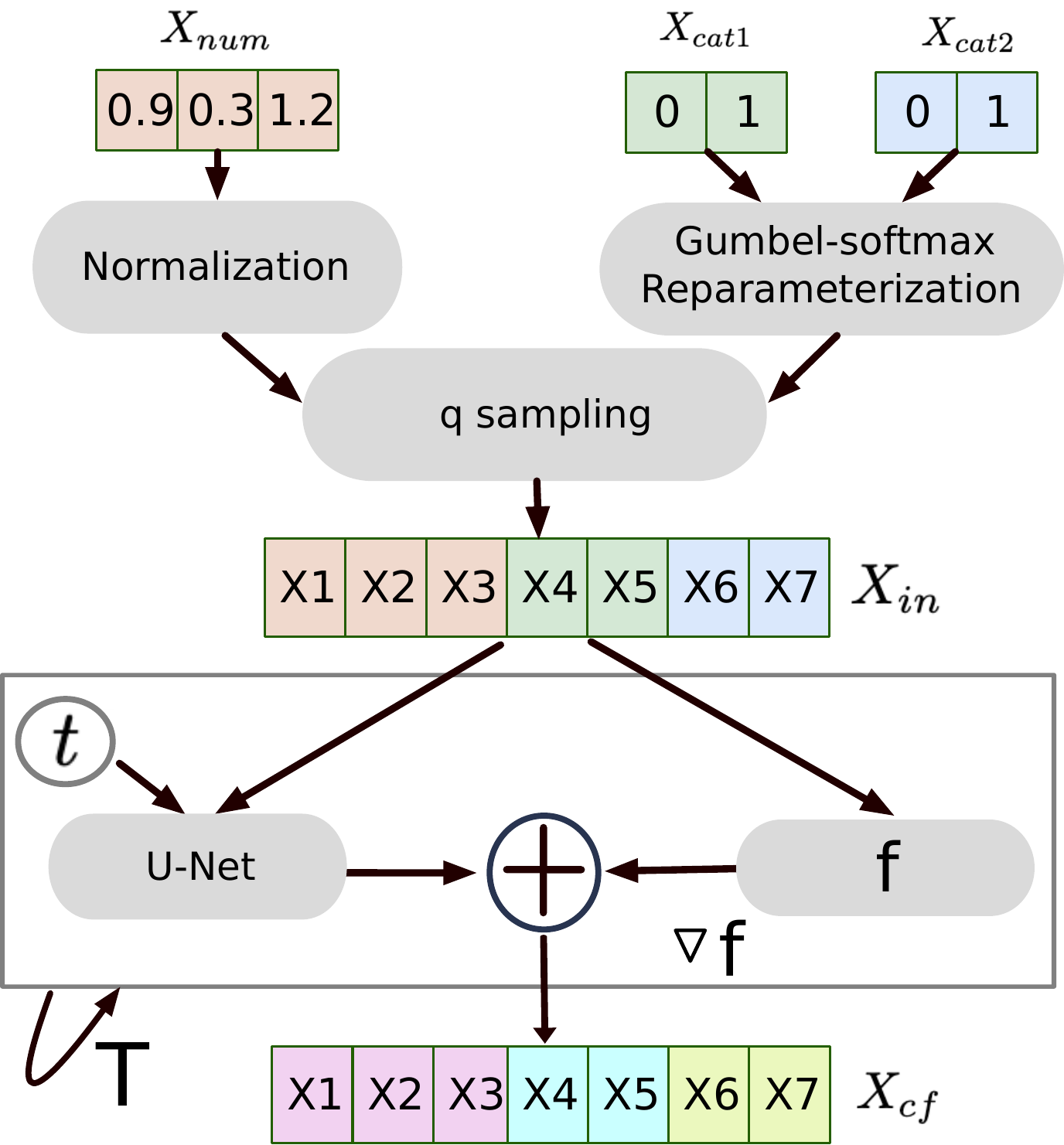}
    \caption{The pipeline of Tabular Diffusion Counterfactual Explanations (TDCE). The categorical variables in the one-hot vector are first re-parameterized. Then, the q sampling generates the noisy version of the input sample. The denoising module runs T steps with the gradient from the classifier to generate the counterfactual sample.  }
    \label{fig:pipeline}
\end{figure}

\section{Categorical Tabular Diffusions for Counterfactual Explanations}
We propose a novel method to generate counterfactual explanations for any differentiable classifier, with particular interest in the categorical data scenario. We adopt the Gumbel-softmax re-parameterization \cite{jang2016categorical} transform to provide a continuous representation of discrete data. This allows the model to leverage the gradients from the differentiable classifier on all the categorical variables and  produce counterfactual information. The pipeline of our method is shown in Figure \ref{fig:pipeline}.

\begin{figure*}[ht]
\centering
  \includegraphics[trim={0 0 0 0cm},clip,width=1\textwidth]{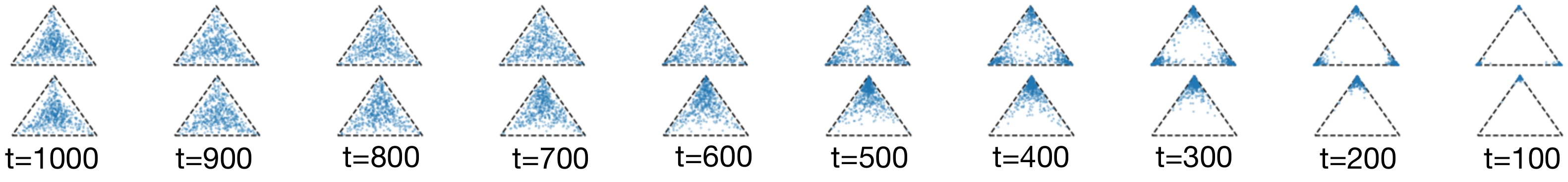}
  \caption{ A simulation of diffusions for the Gumbel-softmax vector over a single categorical variable with three classes. Each blue dot is a data point. Top: The reverse diffusion process for the Gumbel-softmax vector on a 3D simplex \textit{without} classifier guidance. Bottom: The reverse diffusion process \textit{with} classifier guidance. }
  \label{fig:cat-3-sim}
\end{figure*}

\subsection{Tabular counterfactual generation}
We break a data point $x$ into its continuous and categorical portions, $x^{\mathrm{num}}$ and $x^{\mathrm{cat}}$, respectively. 

\paragraph{Continuous features.} Here, we follow the adaptive parameterization \cite{augustin2022diffusion} to implement the guided reverse process. The mean transition of Equation \ref{eqn:cls_guid_rev} now becomes
\begin{equation}\label{eqn:cont_guided}
    \mu_{\theta}(x_t^{\mathrm{num}},t) + \Sigma_\theta(x_t^{\mathrm{num}},t)\|\mu_{\theta}(x_t^{\mathrm{num}},t)\| g_{\mathrm{guided}}
\end{equation}
$$    g_{\mathrm{guided}} = \frac{\nabla \log p_{\phi}(y|f_{dn}(x_t))}{\|\nabla \log p_{\phi}(y|f_{dn}(x_t))\|} - \frac{\nabla d(x, f_{dn}(x_t))}{\|\nabla d(x, f_{dn}(x_t))\|}$$
is the normalized gradient from the classifier and the normalized distance constraint. 

Intuitively, each original denoising step is guided by the classifier's gradients multiplied by the covariance of the denoising step and the magnitude of the unguided mean vector. This process takes the classifier's impact into account and generates high-quality data not only around a dataset's manifold but also in the cluster of the target class. The proposed counterfactual changes should be minimal compared with the initial sample. 

\paragraph{Categorical features.} Working with Equation (\ref{eqn:rev_joint}),
\begin{align}\label{eqn:cat_guid_rev}
    \log p_{\theta,\phi}(x_t|x_{t+1},y) & =  \log p_{\theta}(x_t|x_{t+1}) + \log p_{\phi}(y|f_{dn}(x_t)) - \log Z,
\end{align}
we observe that the adaptive parameterization approach cannot be straightforwardly applied because the gradient cannot be back-propagated to the discrete one-hot vector space. To guide the reverse process in discrete data scenarios, all combinations must be exhausted, which is infeasible and motivated recent developments \cite{schiff2024simple}. 

In this paper, we approach this problem through the Gumbel-softmax re-parameterization; instead of working in the discrete space, we propose to use the Gumbel-softmax vector to softly approximate the discrete data. 

\subsection{Relaxation of categorical variables}
At each time step, a categorical variable is modeled as $x^{\mathrm{cat}} \sim \mathrm{Cat}(x^{\mathrm{cat}} | \overline{\pi})$ where $\overline{\pi} \in \Delta^{K-1}$ is a normalized nonnegative vector. A one-hot vector can be constructed as $x^{\mathrm{cat}} = \mathrm{onehot}(\argmax_i g_i + \log\pi_i )$, where $g_i\sim \mathrm{Gumbel}(0,1)$. Following \citet{jang2016categorical}, and re-parameterize this as
\begin{equation}\label{eqn:re-param}
    \widetilde{x}_{i,t}^{\mathrm{cat}} = \frac{\exp(\frac{1}{\tau}(g_i+\log\overline{\pi}_{i,t} ))}{\sum_{j=1}^K\exp(\frac{1}{\tau}(g_j+\log\overline{\pi}_{j,t} ))}
\end{equation}
at each time step of the reverse process, where $\tau\geq0$ is the temperature. As is evident, as $\tau \rightarrow 0$, $\widetilde{x}_t^{\mathrm{cat}}$ reduces to a one-hot vector. Using this continuous transformation, the $\log p_{\theta}(x_t|x_{t+1})$ term in Equation \ref{eqn:cat_guid_rev} can be modeled with a Gumbel-softmax vector. The density of Gumbel-softmax (GS) \cite{jang2016categorical,maddison2016concrete} is
\begin{equation}\label{eqn:gumbel_softmax}
    p_{GS}(\widetilde{x}_{1:K}|\overline{\pi},\tau) = \Gamma(K)\tau^{K-1}\bigg(\sum_{i=1}^K\frac{\overline{\pi}_i}{\widetilde{x}_i^{\tau}}\bigg)^{-K}\prod_{i=1}^K \frac{\overline{\pi}_i}{\widetilde{x}_i^{\tau+1}}.
\end{equation}
Using Equation \ref{eqn:re-param}, we switch from the discrete one-hot representation to the continuous softmax representation. In the forward and backward process, the transitions are
\begin{align}
    q(\widetilde{x}_t | \widetilde{x}_{t-1}) &\sim \mathrm{GS}\big(\widetilde{x}_t |\overline{\pi}=(1-\beta_t)\widetilde{x}_{t-1} + \beta_t/K\big), \nonumber\\
    q(\widetilde{x}_{t-1}|\widetilde{x}_t,\widetilde{x}_0) &\sim \textstyle \mathrm{GS}\big(\widetilde{x}_{t-1}|\overline{\pi}=\widetilde{\pi}/\sum_{i=1}^K\widetilde{\pi}_i \big),
\end{align}
where $\widetilde{\pi} = [\alpha_t \widetilde{x}_t + (1-\alpha_t)/K]\odot [\overline{\alpha}_{t-1}\widetilde{x}_0 + (1-\overline{\alpha}_{t-1})/K]$. The final categorical sample can be obtained by $x^{\mathrm{cat}} = \mathrm{onehot}(\argmax_i \widetilde{x})$. Equation \ref{eqn:cat_guid_rev} in the Gumbel-softmax space reflects this change straightforwardly,
\begin{align}\label{eqn:rev_gum}
    \log p_{\theta,\phi}(\widetilde{x}_t|\widetilde{x}_{t+1},y)& =  \log p_{\theta}(\widetilde{x}_t|\widetilde{x}_{t+1}) + \log p_{\phi}(y|f_{dn}(\widetilde{x}_t)) + const.
\end{align}
The reverse process $p_{\theta}(\widetilde{x}_t|\widetilde{x}_{t+1})$ is a parameterized neural network. A challenge arises while solving the guided process with the Gumbel-softmax distribution not faced by Gaussian diffusions because the Gaussian model mathematically accommodates a first order Taylor approximation of the classifier well. Therefore, for the Gumbel-softmax we approximate the log density $\log p_{\theta}(\widetilde{x}_t|\widetilde{x}_{t+1})$ as
\begin{equation}\label{eqn:cat_approx}
    \log p_{\theta}(\widetilde{x}_t|\widetilde{x}_{t+1})  ~\approx~ \widetilde{x}_{t}^\top\log\overline{\pi}_{\theta}(\widetilde{x}_{t+1}) +const
\end{equation}
We evaluate this approximation in Section \ref{sec.closeness} in terms of a KL divergence bound between the Gumbel-softmax distribution and Equation \ref{eqn:cat_approx}.

Next, we take the first order of Taylor expansion for the classifier around $\widetilde{x}_{t+1}$, leading to
\begin{equation}\label{eqn:cls_taylor}
    \log p_{\phi}(y | \widetilde{x}_t) \approx (\widetilde{x}_t - \widetilde{x}_{t+1})^\top g_{cat}+ const
\end{equation}
where $g_{cat} = \nabla \log p_{\phi}(y | \widetilde{x}_t)|_{\widetilde{x}_t=\widetilde{x}_{t+1}}$. Replacing Equations \ref{eqn:cat_approx} \& \ref{eqn:cls_taylor} in Equation \ref{eqn:rev_gum}, the guided reverse process becomes
\begin{equation}
    \log p_{\theta,\phi}(\widetilde{x}_t|\widetilde{x}_{t+1},y) \approx \widetilde{x}_t^\top (\log\overline{\pi}_{\theta}(\widetilde{x}_{t+1}) + \lambda g_{cat}) + const,
\end{equation}
where $\lambda$ is a regularization hyperparameter. The familiar expression that results has a similar interpretation to the continuous case. 
We illustrate the reverse process dynamics of our approach in Figure \ref{fig:cat-3-sim}. 

\subsection{Closeness of the approximation}\label{sec.closeness}
At each reverse time step, the log density $\log p_{\theta}(\widetilde{x}_t|\widetilde{x}_{t+1})$ follows the Gumbel-softmax distribution $p_{GS}(\widetilde{x}_t|\widetilde{x}_{t+1})$. We model the log density as
\begin{equation}\label{eqn:approx_gs}
    p_{\theta}(\widetilde{x}_t|\widetilde{x}_{t+1}) = \frac{1}{Z(\widetilde{x}_{t+1})}\prod_i^K \overline{\pi}_{\theta}(\widetilde{x}_{t+1})_i^{\widetilde{x}_{t,i}}
\end{equation}
where $Z(\widetilde{x}_{t+1})$ is the normalizing constant and $\overline{\pi}_{\theta}(\cdot)$ is the probability estimator parameterized by a diffusion network.

\begin{theorem}\label{Thm:thm_kl}
Let $\widetilde{x}, \pi\in \Delta^{K-1}$ and the temperature $\tau\in \mathbb{R}^+$. Define $\widetilde{x}_{min}$ the minimum value $\widetilde{x}$ can take. The KL divergence between $p_{GS}$ defined in Equation \ref{eqn:gumbel_softmax} and its approximation $p_{\theta}$ in Equation \ref{eqn:approx_gs} is bounded as follows:
\begin{align}
    \mathrm{KL}(p_{GS}\|p_{\theta}) < &  -K(\tau+1)\log [1-\widetilde{x}_{min}] + (K-1)\log\tau \nonumber \\
    & + (K-1)\log[1-\widetilde{x}_{min}]  \nonumber \\
    &+ \log\Gamma(K) + K\log[(1-\widetilde{x}_{min})/(K-1)!]\nonumber\\
    \mathrm{KL}(p_{GS}\|p_{\theta}) > & ~  K(\tau+1)\log\widetilde{x}_{min} + (K-1)\log\tau\nonumber\\
    & + (K-1)\log[\widetilde{x}_{min}] \nonumber \\
    &+ \log\Gamma(K) + K\log[1/(K-1)!] \nonumber
\end{align}
\end{theorem}
Proof: See the supplementary material. \newline

\begin{figure}[t]
    \centering
    \includegraphics[width=0.7\linewidth]{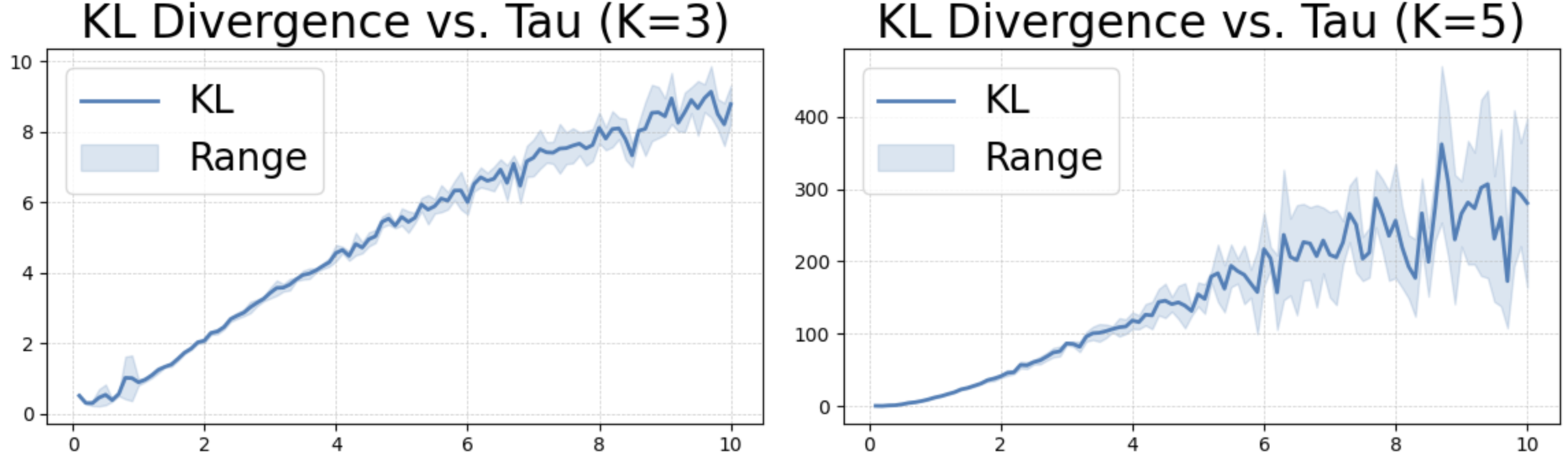}
    \caption{KL divergence between the Gumbel-softmax distribution and our approximation on simulated data as a function of temperature $\tau$. The KL divergence increases as the $\tau$ grows as do the bounds.}
    \label{fig:kl_bound}
\end{figure}

An empirical example of the bound is shown in Figure \ref{fig:kl_bound}. The benefits of our proposed approximation are: \vspace{-0pt}
\begin{enumerate} 
    \item[1)] It allows us to use the first order Taylor expansion, resulting in a closed-form update at each time step for the reverse process. This update directs the unguided logits with the gradient of classifier towards the target class, which is intuitive and similar to the Gaussian case.
    \item[2)] The Taylor expansion, which is also applied to the Gaussian case, requires no additional step for guided categorical variables. The gradient can be calculated concurrently with the Gaussian case over any continuous variables in the data, which significantly reduces computational complexity. 
\end{enumerate}

Although a lower temperature leads to a better approximation, it also introduces a larger variance and may result in vanishing gradient issues. As $\tau\rightarrow 0$, the soft representation approaches a one-hot vector, which may prevent backwards flow of the gradients through the softmax function. A lower temperature can also introduce significant variance of the estimated gradients. To see this, let $Y_i = \mathrm{softmax}((z_i+g_i)/\tau)$ where $g_i$ is the Gumbel noise. The partial derivative $\partial Y_i/\partial z_j = \frac{1}{\tau}Y_i(\delta_{ij}-Y_j)$, which is bounded above by $\frac{1}{4\tau}$, and so $\mathrm{Var}(\partial Y_i/\partial z_j) \leq \mathbb{E}[(\partial Y_i/\partial z_j)^2] < \frac{1}{16\tau^2}$. For the lower bound, the integral over the Gumbel noise is required, which is complicated. 
However, we know that $\partial Y_i/\partial z_j \geq \frac{B}{\tau}$ for some constant $B$, meaning both bounds of the variance indicate that it becomes larger as the temperature decreases. In implementation, we therefore start with a warmer temperature and gradually decrease to a smaller value away from zero.

\begin{algorithm}[t!]
\caption{Tabular diffusion counterfactual explanations}\label{alg:tabular_diff_ce}
\begin{algorithmic}[1]
\Require Input $x$ of continuous $x^{\mathrm{num}}$ and categorical features $x^{cat}$. Binary mask $m$ of immutable features in $x$. Classifier $f$ and denoiser $f_{dn}$
\State $x_t^{\mathrm{num}}\sim  \mathcal{N}(x_t | \sqrt{\overline{\alpha}_t}x_0,\sqrt{1 - \overline{\alpha}_t} I)$ 
\Statex $\widetilde{x}_t^{\mathrm{cat}}\sim \mathrm{GS}(\widetilde{x}_t|\overline{\alpha}_t\widetilde{x}_0 + (1-\overline{\alpha}_t)/K)$
\For{iteration $t = T,\dots,0$}
\State $g^{\mathrm{num}} = \frac{\nabla \log p_{\phi}(y|f_{dn}(x_t))}{\|\nabla \log p_{\phi}(y|f_{dn}(x_t))\|}$
\State $g^{\mathrm{cat}}~~ = \nabla \log p_{\phi}(y |\widetilde{x}_t)|_{\widetilde{x}_t=\widetilde{x}_{t+1}}$
\State $\mu, \Sigma \,\,\leftarrow \mu_{\theta}([x_t^{\mathrm{num}},\widetilde{x}_t^{\mathrm{cat}}]), \Sigma_{\theta}([x_t^{\mathrm{num}},\widetilde{x}_t^{\mathrm{cat}}])$
\State $x^{\mathrm{num}}\leftarrow \mu^{\mathrm{num}} + \|\mu^{\mathrm{num}}\|\Sigma^{\mathrm{num}} g^{\mathrm{num}}$
\State $\widetilde{x}^{\mathrm{cat}} ~\,\leftarrow \mu^{\mathrm{cat}} + \|\mu^{\mathrm{cat}}\| g^{\mathrm{cat}}$
\State $x_{t-1}\, \leftarrow [x^{\mathrm{num}},\widetilde{x}^{\mathrm{cat}}]\odot m + [x_t^{\mathrm{num}},\widetilde{x}_t^{\mathrm{cat}}]\odot (1-m)$
\EndFor
\end{algorithmic}
\end{algorithm}

\subsection{Immutable features}
Immutable features are those that are predefined as unchangeable by the source of a datum, e.g., a location. When generating counterfactuals, these cannot be changed. One simple approach is to define a binary mask $m$ indicating which features can change, and produce the counterfactual $x_t*m + x*(1-m)$. The main issue here is with so-called coherence \cite{avrahami2022blended}, yielding samples that fall outside the data manifold.

Motivated by the blended diffusions of vision tasks \cite{avrahami2022blended}, we combine the noisy version of the immutable features from the input with the guided mutable features according to $x_{t,\mathrm{guided}}*m + x_{t,\mathrm{noisy}}*(1-m)$ where $x_{t,\mathrm{noisy}}$ is obtained from the forward process. At the final step, the immutable features are replaced by the original input. Our algorithm is shown in Algorithm \ref{alg:tabular_diff_ce}.

\section{Experiments}
We compare our method with other popular methods for generating counterfactual explanations. The classifier $f$ to be explained shares the same architecture as the U-net in the diffusion model and the last layer outputs two dimensional logits for binary classification. We refer to our method as Tabular Diffusion Counterfactual Explanation (TDCE).

\subsection{Datasets}
We focus on tabular datasets, selecting popular and public datasets that consists of both numerical and categorical features. A description of data is shown in Table \ref{table:datasets}. Lending Club Dataset (LCD) and Give Me Some Credit (GMC) focus on credit lending decisions. ``Adult'' predicts binarized annual income based on a set of features. The LAW data predicts pass/fail on a law school test. In all the experiments, we perform standardization for each continuous feature and convert each categorical feature to one-hot vector.

\begin{table}[t]
\centering 
\resizebox{0.6\columnwidth}{!}{
\begin{tabular}{ c c c c c c }
 \hline
 \textbf{Dataset} & \textbf{$\#$Train} & \textbf{$\#$Val} & \textbf{$\#$Test} & \textbf{$\#$Num}& \textbf{$\#$Cat}\\ 
 \hline
  LCD             & 10,000 & 1,000 & 1,000 & 5 & 1\\
  GMC             & 15,000 & 1,000 & 1,000 & 9  & 1\\
  Adult           & 47,842 & 1,000 &1,000  & 9 & 2\\
  LAW             & 5,502 & 1,000 & 1,000 & 8 & 3\\
 \hline
\end{tabular}
}
\caption{Statistics from the tabular data sets we use.}\label{table:datasets}
\end{table} 

\subsection{Baselines and evaluation metrics}
As a baseline, we compare with five methods for generating counterfactual explainations of a binary classifier. \citet{wachter2017counterfactual} present the most straightforward baseline. They generate counterfactuals by following the gradients of a classifier from the input $x$ to the decision boundary. However, as we show in the next section, this simple and intuitive approach struggles to generate realistic counterfactuals. We also benchmark against VAE-based methods designed to fix this, including CCHVAE \cite{pawelczyk2020learning}, REVISE \cite{joshi2019towards}, and CLUE \cite{antoran2020getting}, as well as a method based on graph search called FACE \cite{poyiadzi2020face} and the neural network based method CounterNet \cite{guo2023counternet}. We implement these benchmarks using the CARLA library \cite{pawelczyk2021carla} with their default parameterizations.


We evaluate these methods using several widely used metrics for counterfactual-based explainability: Interpretability, Diversity, Validity, Instability and the JS divergence. Among these, Diversity and Instability are restricted to continuous features, while JS divergence applies only to categorical features. There is no global metric that quantifies counterfactual performance, and so the set of metrics described below combines to paint a subjective picture for evaluation.

\paragraph{\textbf{L2 Distance.}} Counterfactual samples aim to have turned their label to the target class with the minimum changes in the feature space. This is the key standard of counterfactual generation described in Equation \ref{eqn:cf_obj} and can be quantified using L2 distance,
\begin{equation}
    \text{L2} = \frac{1}{N}\sum_{i=1}^N ||x_i - x_i^{\mathrm{cf}}||_2^2,
\end{equation}
Please note that this metric can only evaluate continuous features. For categorical features, we aim to recover the distribution of the categorical variable in the target class, which will be described later.

\paragraph{\textbf{Interpretability.}} \citet{van2021interpretable} use an autoencoder to evaluate the \textit{interpretability} of a counterfactual method. Let $\mathrm{AE}_o$, $\mathrm{AE}_t$, and $\mathrm{AE}$ be three autoencoders trained on the original class, target class, and the entire dataset, respectively. The IM1 and IM2 scores are
\begin{align}
    \mathrm{IM1} &=  \frac{1}{N}\sum_{i=1}^N\frac{\|x_i^{\mathrm{cf}} - \mathrm{AE}_t(x_i^{\mathrm{cf}})\|^2}{\|x_i^{\mathrm{cf}} - \mathrm{AE}_o(x_i^{\mathrm{cf}})\|^2 + \epsilon} \nonumber \\
    \mathrm{IM2} & =   \frac{1}{N}\sum_{i=1}^N\frac{\|\mathrm{AE}_t(x_i^{\mathrm{cf}}) - \mathrm{AE}(x_i^{\mathrm{cf}})\|^2}{\|x_i^{\mathrm{cf}}\|_1 + \epsilon}
\end{align}
where $x_i^{\mathrm{cf}}$ is the $i$th of $N$ counterfactuals. Lower value of IM1 indicates that the generated counterfactuals are reconstructed better by the autoencoder trained on the counterfactual class ($\mathrm{AE}_t$) than the autoencoder trained on the original class. This suggests that the counterfactual is closer to the data manifold of the counterfactual class, and thus more plausible. A similarly interpretation holds for IM2. Hence, lower values of IM1 and IM2 are preferred.

\paragraph{\textbf{Diversity.}} Diversity provides additional performance information because low IM1 and IM2 may occur with counterfactuals that tend to merge to a single point; not only should the counterfactual look like the counterclass, it should also preserve its variety. The diversity metric is calculated as
\begin{equation}
    \text{Diversity} =  \frac{1}{N(N-1)}\sum_{i=1}^N\sum_{j=i+1}^N d(x_i^{\mathrm{cf}},x_j^{\mathrm{cf}}),
\end{equation}
where $d(\cdot, \cdot)$ is a predefined distance function. We use the Euclidean distance in this paper. 

\paragraph{\textbf{Validity.}} This metric verifies that the generated counterfactual indeed lies in the counter-class region of the classifier to be explained. This is
\begin{equation}
    \text{Validity} =  \frac{1}{N}\sum_{i=1}^N \mathbb{1}(f(x_i^{\mathrm{cf}}) = y')
\end{equation}
where $f(\cdot)$ is the explained classifier and $y'$ is the target label. (Not all counterfactual methods generate counterfactuals that are guaranteed to change their label.)

\paragraph{\textbf{Instability.}} A stable counterfactual explainer should produce similar counterfactual outputs for two similar query inputs. Instability quantifies this as
\[
\text{Instability} =  \frac{1}{N} \sum_{i=1}^N \frac{d(x_i^{\mathrm{cf}},\widehat{x}_i^{\mathrm{cf}})}{1 + d(x_i,\widehat{x}_i)}
\]
where $\widehat{x}_i = \arg\min_{x\in X\setminus x_i, f(x) = f(x_i)} \|x-x_i\|$, the point within the data set closest to $x_i$ that has the same label. A low instability is preferred.

\paragraph{\textbf{JS Divergence.}} We also evaluate how well the distribution of counterfactual categorical variables aligns with the distribution of the target class. We calculate the average JS divergence across categorical variables,
\begin{equation}
    \text{JS} =  \frac{1}{N^c}\sum_{i=1}^{N^c} \mathrm{JS}(P_{\text{target}}(x_i) \| P_{CF}(x_i)) 
\end{equation}
where $N^c$ is the number of categorical variables. A lower JS score indicates similarity between the distributions of generated counterfactuals and the target class.

\subsection{Results}
\begin{table*}[t!]
    \centering
    \resizebox{0.95\textwidth}{!}{
    \begin{tabular}{c c|c|c|c|c|c|c|c }
        \hline\hline
        \multicolumn{9}{c}{\textbf{Counterfactual Evaluations}}\\
        \hline
        \multicolumn{2}{c|}{\textbf{Model}} & L2$\downarrow$& Diversity$\uparrow$  & Instability$\downarrow$  & JS$\downarrow$  & IM1$\downarrow$& IM2$\downarrow$  & Validity$\uparrow$ \\
        \hline
        \multirow{7}{*}{\rotatebox[origin=c]{90}{\textbf{LCD}}}
        & Wach. &$\textbf{0.34} {\pm 0.02}$& $\textbf{0.73} {\pm 0.03}$ & $ \underline{0.11} {\pm 0.03}$ & $0.12 {\pm 0.03}$& $1.33 {\pm 0.04}$ & $0.16 {\pm 0.03}$& $0.60 {\pm 0.03}$ \\
        & CCH. & $0.56 {\pm 0.03}$& $0.19 {\pm 0.01}$ & $0.21 {\pm 0.02}$ & $0.09 {\pm 0.01}$ & $\textbf{0.57} {\pm 0.01}$ & $\underline{0.08} {\pm 0.01}$ & $\textbf{0.99} {\pm 0.01}$ \\
        & REVI. &$0.59 {\pm 0.01}$& $0.18 {\pm 0.03}$ & $0.22 {\pm 0.02}$ & $0.10 {\pm 0.01}$ & $0.89 {\pm 0.03}$ &$0.09 {\pm 0.02}$ & $\textbf{0.99} {\pm 0.01}$\\
        & CLUE   &$0.70 {\pm 0.02}$ & $0.26 {\pm 0.03}$ & $0.31 {\pm 0.03}$& $0.11 {\pm 0.01}$& $0.72 {\pm 0.04}$& $0.11 {\pm 0.01}$ & $0.83 {\pm 0.03}$ \\
        & FACE   &$0.69 {\pm 0.01}$& $\underline{0.54} {\pm 0.05}$ & $\underline{0.11} {\pm 0.01}$ & $\underline{0.06} {\pm 0.01}$ & $0.91 {\pm 0.07}$& $0.11 {\pm 0.03}$ & $0.85 {\pm 0.02}$\\
        & CounterNet & $\underline{0.35} {\pm 0.01}$ & $0.45 {\pm 0.03}$ & $0.25 {\pm 0.02}$ & $0.15 {\pm 0.02}$ & $0.99 {\pm 0.03}$ & $0.69 {\pm 0.03}$ & $\textbf{0.99} {\pm 0.01}$ \\
        & {TDCE}  &$0.59 {\pm 0.03}$&$\textbf{0.73} {\pm 0.03}$ & $ \textbf{0.05} {\pm 0.01}$ & $\textbf{0.01} {\pm 0.01}$ & $\underline{0.63} {\pm 0.03}$ & $\textbf{0.05} {\pm 0.01}$ & $\textbf{0.99} {\pm 0.01}$ \\
        \hline
        \multirow{7}{*}{\rotatebox[origin=c]{90}{\textbf{GMC}}}
        & Wach. &$\textbf{0.03} {\pm 0.02}$& $\underline{0.25} {\pm 0.02}$ & $0.09 {\pm 0.01}$ & $\textbf{0.03} {\pm 0.01}$ & $\underline{1.04} {\pm 0.05}$ & $\textbf{0.07} {\pm 0.01}$ & $0.73 {\pm 0.03}$ \\
        & CCH.  &$0.21 {\pm 0.03}$& $0.21 {\pm 0.01}$ & $0.10 {\pm 0.01}$ & $0.06 {\pm 0.02}$ & $1.14 {\pm 0.05}$ & $0.15 {\pm 0.02}$ & $0.77 {\pm 0.02}$\\
        & REVI.  &$0.23 {\pm 0.02}$ & $0.21 {\pm 0.02}$ & $0.13 {\pm 0.01}$ & $\underline{0.05} {\pm 0.01}$& $1.18 {\pm 0.05}$ & $\textbf{0.07} {\pm 0.01}$ & $0.80 {\pm 0.02}$ \\
        & CLUE  & $\underline{0.18} {\pm 0.02}$& $0.18 {\pm 0.02}$ & $\underline{0.07} {\pm 0.01}$ & $0.08 {\pm 0.01}$ & $1.14 {\pm 0.04}$& $\textbf{0.07} {\pm 0.01}$ & $0.81 {\pm 0.02}$ \\
        & FACE  & $0.21 {\pm 0.02}$& $0.17 {\pm 0.02}$ & $\textbf{0.05} {\pm 0.01}$ & $0.07 {\pm 0.01}$ & $1.18 {\pm 0.01}$&$0.08 {\pm 0.01}$ & $0.86 {\pm 0.01}$ \\
        & CounterNet & $0.20 {\pm 0.01}$ & $0.17 {\pm 0.02}$ & $0.10 {\pm 0.01}$ & $0.06 {\pm 0.01}$ & $1.02 {\pm 0.02}$ & $0.11 {\pm 0.02}$ & $\underline{0.97} {\pm 0.01}$ \\
        & TDCE & $\underline{0.18} {\pm 0.03}$ & $\textbf{1.08} {\pm 0.06}$& $\textbf{0.05} {\pm 0.01}$ & $\textbf{0.03} {\pm 0.01}$  & $\textbf{0.96} {\pm 0.04}$ & $\textbf{0.06} {\pm 0.02}$ & $\textbf{0.99} {\pm 0.01}$ \\
        \hline
        \multirow{7}{*}{\rotatebox[origin=c]{90}{\textbf{Adult}}}
        & Wach.& $\textbf{0.27} {\pm 0.03}$ & $\textbf{1.11} {\pm 0.01}$ & $0.09 {\pm 0.01}$ & $0.13 {\pm 0.01}$ & $1.31 {\pm 0.03}$ &$\textbf{0.05} {\pm 0.01}$ & $0.57 {\pm 0.02}$ \\ 
        & CCH. & $\underline{0.79} {\pm 0.03}$& $0.19 {\pm 0.02}$ & $0.22 {\pm 0.02}$ & $0.11 {\pm 0.02}$ & $1.89 {\pm 0.07}$ &$0.06 {\pm 0.02}$ & $0.61 {\pm 0.03}$ \\
        & REVI. & $0.99 {\pm 0.02}$& $0.43 {\pm 0.02}$ & $ 0.10 {\pm 0.01}$ & $0.11 {\pm 0.01}$ & $1.11 {\pm 0.01}$&$0.07 {\pm 0.01}$ & $0.58 {\pm 0.02}$\\
        & CLUE  & $0.81 {\pm 0.03}$& $0.11 {\pm 0.01}$ & $\textbf{0.04} {\pm 0.01} $ & $0.17 {\pm 0.03}$ & $1.41 {\pm 0.05}$&$\textbf{0.04} {\pm 0.01}$ & $0.62 {\pm 0.01}$\\
        & FACE  & $0.89 {\pm 0.02}$& $0.74 {\pm 0.04}$ & $\underline{0.07} {\pm 0.01}$ & $\underline{0.06} {\pm 0.01}$ & $0.97 {\pm 0.02}$ &$0.06 {\pm 0.01}$ & $\underline{0.75} {\pm 0.02}$ \\
        & CounterNet & $0.86 {\pm 0.02}$ & $0.69 {\pm 0.02}$ & $\underline{0.07} {\pm 0.02}$ & $0.09 {\pm 0.01}$ & $\underline{0.96} {\pm 0.02}$ & $0.06 {\pm 0.01}$ & $\textbf{0.94} {\pm 0.01}$\\
        & TDCE & $0.85 {\pm 0.04}$& $\underline{0.80} {\pm 0.03}$ & $\textbf{0.05} {\pm 0.01}$ & $\textbf{0.03} {\pm 0.02}$ & $\textbf{0.90} {\pm 0.02}$ & $\textbf{0.04} {\pm 0.01}$ & $\textbf{0.94} {\pm 0.04}$ \\
        \hline
        \multirow{7}{*}{\rotatebox[origin=c]{90}{\textbf{LAW}}}
        & Wach. & $\textbf{0.17} {\pm 0.04}$& $\textbf{1.22} {\pm 0.05}$ & $0.13 {\pm 0.02}$ & $0.11 {\pm 0.02}$& $1.73 {\pm 0.02}$ & $0.12 {\pm 0.02}$ & $0.58 {\pm 0.01}$ \\
        & CCH. & $0.99 {\pm 0.02}$& $0.20 {\pm 0.01}$& $0.07 {\pm 0.01}$ & $\underline{0.05} {\pm 0.01}$ & $0.95 {\pm 0.03}$ & $\underline{0.09} {\pm 0.02}$ & $\textbf{0.99} {\pm 0.01}$\\
        & REVI. & $\underline{0.71} {\pm 0.03}$ & $0.91 {\pm 0.03}$ & $\underline{0.06} {\pm 0.01}$ & $0.06 {\pm 0.01}$ & $1.56 {\pm 0.05}$ & $0.11 {\pm 0.01}$ & $0.61 {\pm 0.01}$ \\
        & CLUE & $0.79 {\pm 0.02}$ & $0.37 {\pm 0.01}$ & $0.07 {\pm 0.01}$ & $\underline{0.05} {\pm 0.01}$ & $1.21 {\pm 0.02}$ & $\textbf{0.06} {\pm 0.02}$ & $\textbf{0.99} {\pm 0.01}$ \\
        & FACE & $0.81 {\pm 0.02}$& $0.83 {\pm 0.02}$ & $\textbf{0.03} {\pm 0.01}$ & $\textbf{0.04} {\pm 0.01}$ & $1.31 {\pm 0.06}$ & $0.11 {\pm 0.02}$ & ${0.81} {\pm 0.02}$  \\
        & CounterNet & $0.79 {\pm 0.01}$ & $0.91 {\pm 0.02}$ & $0.07 {\pm 0.01}$ & $0.06 {\pm 0.01}$ & $\underline{0.93} {\pm 0.03}$ & $0.08 {\pm 0.01}$ & $\textbf{0.99} {\pm 0.01}$\\
        & TDCE & $0.81 {\pm 0.02}$ & $\underline{0.97} {\pm 0.03}$ & $\underline{0.06} {\pm 0.02}$ & $\textbf{0.04} {\pm 0.02}$ & $\textbf{0.89} {\pm 0.05}$ & $\textbf{0.06} {\pm 0.01}$& $\textbf{0.99} {\pm 0.01}$ \\
        \hline\hline
    \end{tabular}
    }
    \caption{Counterfactual quantitative evaluation without masking of features that are allowed to change. We provide an evaluation according to the metrics described in the text. The arrow beside each metric indicates the preferred value. We select one feature to mask in the masking setting. (bold = 1st, underline = 2nd). Note: The classifier in CounterNet requires a different architecture because it is model-dependent.}
    \label{tab:cf_eval}
\end{table*}

\begin{table*}[t!]
    \centering
    \resizebox{0.9\textwidth}{!}{
    \begin{tabular}{c c|c|c|c|c|c|c }
        \hline\hline
        \multicolumn{8}{c}{\textbf{Counterfactual Evaluations}}\\
        \hline
        \multicolumn{2}{c|}{\textbf{Model}} & L2$\downarrow$ & Diversity$\uparrow$ & Instability$\downarrow$ & IM1$\downarrow$& IM2$\downarrow$  & Validity$\uparrow$ \\
        \hline
        \multirow{7}{*}{\textbf{LCD}}
        & Wach. &$\textbf{0.34} {\pm 0.03}$& $\underline{0.73} {\pm 0.03}$& $\underline{0.12} {\pm 0.01}$& $1.04 {\pm 0.05}$ & $0.27 {\pm 0.03}$ & $0.75 {\pm 0.03}$\\
        & CCH.  &$0.50 {\pm 0.03}$&  $0.36 {\pm 0.03}$ & $0.29 {\pm 0.02}$  & $\textbf{0.64} {\pm 0.05}$ & $\underline{0.16} {\pm 0.01}$  & $\textbf{0.98} {\pm 0.01}$\\
        & REVI. &$0.52 {\pm 0.02}$&  $0.33 {\pm 0.03}$ & $0.21 {\pm 0.02}$ & $0.82 {\pm 0.04}$ & $0.19 {\pm 0.02}$ & $\textbf{0.98} {\pm 0.01}$\\
        & CLUE   &$0.49 {\pm 0.02}$&  $0.38 {\pm 0.04}$ & $0.24 {\pm 0.02}$ & $0.92 {\pm 0.02}$ & $\underline{0.15} {\pm 0.01}$ & $0.81 {\pm 0.02}$\\
        & FACE   & $0.69 {\pm 0.02}$ & $0.55 {\pm 0.03}$ & $0.17 {\pm 0.01}$ & $0.79 {\pm 0.07}$ & $0.20 {\pm 0.01}$ & $0.87 {\pm 0.01}$\\
        & CounterNet & $\underline{0.35} {\pm 0.01}$ & $0.45 {\pm 0.03}$ & $0.25 {\pm 0.02}$ &  $1.09 {\pm 0.03}$ & $0.88 {\pm 0.03}$ & $\textbf{0.99} {\pm 0.01}$ \\
        & {TDCE}  &$0.49 {\pm 0.02}$&$\textbf{0.77} {\pm 0.03}$ &$\textbf{0.09} {\pm 0.02}$ &$\underline{0.77} {\pm 0.02}$ &$\textbf{0.06} {\pm 0.02}$ & $\textbf{0.99} {\pm 0.01}$\\
        \hline
        \multirow{7}{*}{\textbf{GMC}}
        & Wach. &$\textbf{0.04} {\pm 0.01}$& $\underline{0.23} {\pm 0.02}$ & $0.10 {\pm 0.01} $  & $1.13 {\pm 0.09}$ & $\underline{0.13} {\pm 0.02}$ & $0.57 {\pm 0.03}$\\
        & CCH.  & $0.17 {\pm 0.02}$& $0.21 {\pm 0.01}$ & $0.11 {\pm 0.01}$ & $1.19 {\pm 0.03}$ &  $0.15 {\pm 0.01}$  & $0.52 {\pm 0.02}$ \\
        & REVI.  & $0.16 {\pm 0.02}$& $0.21 {\pm 0.02}$ & $0.12 {\pm 0.01}$  & $1.10 {\pm 0.05}$ & $0.17 {\pm 0.01}$ & $0.53 {\pm 0.02}$\\
        & CLUE   & $0.11 {\pm 0.02}$& $0.20 {\pm 0.02}$ & $\underline{0.08} {\pm 0.01}$ & $1.32 {\pm 0.05}$ & $\underline{0.13} {\pm 0.01}$ & $0.57 {\pm 0.01}$ \\
        & FACE   &$\underline{0.09} {\pm 0.03}$&  $0.16 {\pm 0.02}$ & $\textbf{0.07} {\pm 0.02}$ & $\underline{1.03} {\pm 0.02}$ & $\underline{0.13} {\pm 0.02}$ & $0.65 {\pm 0.02}$\\
        & CounterNet & $0.20 {\pm 0.01}$ & $0.17 {\pm 0.02}$ & $0.10 {\pm 0.01}$ & $1.09 {\pm 0.02}$ & $0.16 {\pm 0.02}$ & $\underline{0.90} {\pm 0.01}$\\
        & TDCE   &$0.11 {\pm 0.02}$& $\textbf{0.83} {\pm 0.03}$ & $\textbf{0.06} {\pm 0.01}$ & $\textbf{0.99} {\pm 0.03}$ & $\textbf{0.05} {\pm 0.01}$ & $\textbf{0.94} {\pm 0.02}$ \\
        \hline
        \multirow{7}{*}{\textbf{Adult}}
        & Wach. & $\textbf{0.28} {\pm 0.04}$& $\textbf{1.01} {\pm 0.03}$ & $0.15 {\pm 0.01}$  & $\underline{1.00} {\pm 0.05}$ & $0.07 {\pm 0.01}$  & $0.51 {\pm 0.02}$\\ 
        & CCH. & $0.62 {\pm 0.03}$& $0.72 {\pm 0.03}$ & $0.17 {\pm 0.02}$  & $1.11 {\pm 0.03}$ & $0.11 {\pm 0.01}$ & $0.55 {\pm 0.03}$\\
        & REVI.&$0.78 {\pm 0.03}$ & $0.78 {\pm 0.02}$ & $0.09 {\pm 0.01}$  & $1.11 {\pm 0.06}$ & $0.07 {\pm 0.01}$ & $0.61 {\pm 0.02}$\\
        & CLUE   & $\underline{0.61} {\pm0.03}$& $0.71 {\pm 0.03}$ & $\textbf{0.07} {\pm 0.01}$  & $1.14 {\pm 0.03}$ & $\textbf{0.06} {\pm 0.01}$ & $0.55 {\pm 0.01}$\\
        & FACE   & $0.85 {\pm 0.02}$& $0.79 {\pm 0.02}$ & $\underline{0.08} {\pm 0.01}$ & $\underline{1.02} {\pm 0.02}$& $\textbf{0.06} {\pm 0.01}$ & $0.58 {\pm 0.02}$\\
        & CounterNet & $0.86 {\pm 0.02}$ & $0.69 {\pm 0.02}$ & $\underline{0.07} {\pm 0.02}$ & $1.03 {\pm 0.02}$ & $0.10 {\pm 0.01}$ & $\underline{0.84} {\pm 0.01}$\\
        & TDCE  & $0.79 {\pm 0.03}$& $\underline{0.82} {\pm 0.03}$& $\textbf{0.06} {\pm 0.02}$ & $\textbf{0.93} {\pm 0.04}$ & $\textbf{0.05} {\pm 0.02}$ & $\textbf{0.86} {\pm 0.04}$ \\
        \hline
        \multirow{7}{*}{\textbf{LAW}}
        & Wach. &$\textbf{0.26} {\pm 0.02}$ & $\textbf{1.12} {\pm 0.02}$ & $0.13 {\pm 0.02}$ & $1.54 {\pm 0.01}$ & $0.14 {\pm 0.01}$  & $0.39 {\pm 0.03}$\\
        & CCH.  &$0.89 {\pm 0.02}$ & $0.75 {\pm 0.02}$ & $\underline{0.05} {\pm 0.01}$ & $1.43 {\pm 0.02}$ & $0.13 {\pm 0.03}$ & $\textbf{0.99} {\pm 0.01}$\\
        & REVI. & $0.80 {\pm 0.03}$ & $\underline{1.02} {\pm 0.02}$ & $0.09 {\pm 0.02}$ & $1.37 {\pm 0.02}$& $0.11 {\pm 0.01}$ & $0.60 {\pm 0.01}$\\
        & CLUE  &$0.81 {\pm 0.02}$& $0.68 {\pm 0.01}$ & $0.07 {\pm 0.02}$ & $\underline{0.76} {\pm 0.01}$ & $\underline{0.09} {\pm 0.01}$ & $\textbf{0.99} {\pm 0.01}$\\
        & FACE & $0.92 {\pm 0.01}$ & $0.81 {\pm 0.03}$ & $\textbf{0.04} {\pm 0.01}$ & $1.63 {\pm 0.01}$ & $0.16 {\pm 0.01}$  & $0.80 {\pm 0.03}$\\
        & CounterNet & $0.79 {\pm 0.01}$ & $0.91 {\pm 0.02}$ & $0.07 {\pm 0.01}$ & $0.96 {\pm 0.03}$ & $\underline{0.09} {\pm 0.01}$ & $\underline{0.98} {\pm 0.01}$\\
        & TDCE & $\underline{0.79} {\pm 0.02}$&  $0.95 {\pm 0.03}$ & $\underline{0.05} {\pm 0.01}$ & $\textbf{0.73} {\pm 0.03}$ & $\textbf{0.07} {\pm 0.01}$ & $\underline{0.98} {\pm 0.02}$\\
        \hline\hline
    \end{tabular}
    }
    \caption{Counterfactual quantitative evaluation with masking of features that are allowed to change. We provide an evaluation according to the metrics described in the text. The arrow beside each metric indicates the preferred value. We select one feature to mask in the masking setting. (bold = 1st, underline = 2nd)}
    \label{tab:cf_eval_mask}
\end{table*}

\paragraph{Quantitative evaluation.} We show quantitative results in Table \ref{tab:cf_eval} and Table \ref{tab:cf_eval_mask}. The Table \ref{tab:cf_eval} is for the no-masking setting, while the Table \ref{tab:cf_eval_mask} is for masking a preselected feature, prohibiting it from being changed by the counterfactual generator. While there is no single combination of these various metrics that determines relative performance, a subjective evaluation indicates competitive performance of our TDCE method. For example, it achieves the best validity among other methods with significant margins, indicating nearly all the generated samples have turned to the target class -- arguably a prerequisite for other metrics to have meaning. We observe the competitive or superior performance on IM1 and IM2 as well, indicating that the generated counterfactuals stay on the data manifold of the target class and have better interpretability.

In the experiments, the baseline Wachter shows volatile performance in terms of the metrics across different datasets. It is able to produce robust samples with fair diversity, but it lacks interpretability such as on LCD. We also note that CCHVAE, REVISE and CLUE show strong robustness (low instability). However, these results are usually accompanied by a low diversity score, indicating that the algorithms tends to generate similar counterfactuals. The same conclusion can also be drawn from their relatively high JS score, a high JS score suggesting that the generated categorical variables do not match the distribution in the target class. 

To give a concrete example of what we observed, when analyzing the LCD dataset we found that the FICO score dominates the classifier's decision, while the categorical loan-term variable is less important. All benchmark methods tend to completely ignore the less significant feature because changing one's FICO score quickly changes the classification. In contrast, our TDCE method pays more attention to each feature, which creates less discrepancy between the distributions of the counterfactual and counter-class data.
We also note that TDCE is faster compared with other searching algorithms. Wachter, REVISE and FACE require iterative searching for each individual sample, while TDCE's reverse process is a fixed Markov chain learned during training.

\paragraph{Proximity.} We evaluate the proximity between the generated counterfactual sample and the original query samples using L2 distance. In the experiments, we observe that Wachter has achieved the lowest L2 distance across all the datasets. This is because Wachter stops searching when it finds the sample that changes its label with minimum modification in the feature space. This can produce a sample around the decision boundary, though by using minimal changes it might produce less a meaningful explanation that is out of sample (i.e. has high IM1/IM2). We emphasize that we are focusing on a higher validity and greater interpretability (i.e. low IM1/IM2) and therefore are willing to sacrifice on L2 distance. Nevertheless, our TDCE is adjustable for the importance of L2 distance by simply adding a regularization to Equation \ref{eqn:cont_guided}. 

\paragraph{Efficency.} We show the runtime for generating counterfactual samples in Table \ref{tab:cf_speed}. We observe that CounterNet runs faster than other methods because it generates the counterfactuals through a single forward pass of a neural network. However, we highlight that our TDCE is model-agnostic, whereas CounterNet is model-dependent, requiring the encoder to be the classifier as well. In addition, TDCE is capable of generating more interpretable samples (i.e. low IM2/IM2) for all datasets we consider. Aside from CounterNet, TDCE runs faster than all other search-based algorithms because TDCE can generate the counterfactual samples in a pre-defined number of reverse steps, while search algorithms rely heavily on the objective function. 

\begin{table*}[t!]
    \centering
    \resizebox{0.7\textwidth}{!}{
    \begin{tabular}{c |c|c|c|c|c|c| c}
        \hline\hline
        \multicolumn{8}{c}{\textbf{Counterfactual Generation Time (sec/100 sample)}}\\
        \hline
        \textbf{\backslashbox{Dataset}{Model}} & Wach. & CCH.  & REVI. & CLUE & FACE & CounterNet  & {TDCE} \\
        \hline
        \textbf{LCD} & 0.9& 0.8& 0.8& 0.9& 1.0& 0.1& 0.3\\
        \hline 
        \textbf{GMC} & 1.1 & 1.2& 1.1& 1.3& 1.5& 0.2 & 0.5 \\
        \hline
        \textbf{Adult}& 1.1 & 1.3& 1.2& 1.3& 1.6& 0.2& 0.5 \\
        \hline
        \textbf{LAW}& 1.0 & 1.1 & 1.1& 1.2& 1.4 & 0.1& 0.4  \\
        \hline
        \hline
    \end{tabular}
    }
    \caption{Counterfactual generation time in seconds per 100 samples on the same computer setting.}
    \label{tab:cf_speed}
\end{table*}

\paragraph{Discussion.} Through the experiments, each benchmark shows a good ability to generate counterfactuals, yet they are somewhat limited due to design issues. Gradient-based method such as Wachter, which directly operate in the feature space, can often be fooled by spurious changes. This causes the classifier to change its prediction by small, uninterpretable movements. VAE-based methods such as CLUE, REVISE and CCHVAE leverage the generative power of VAEs but heavily rely on the black-box latent space in which they work. This may lead to counterfactuals that fall off the data manifold. Graph methods like FACE depend on sample quality and coverage, and search only within the graph itself, producing a sample that already exists in the data set. In contrast, our TDCE uses a diffusion model operating directly in the ambient features space. This connects it to Wachter, while still leveraging the generative power available to deep models such as the VAE. We believe the combination of these desireable aspects accounts for our good relative performance.

\begin{figure*}[th!]
\centering
  \includegraphics[trim={0 0 0 0cm},clip,width=1\textwidth]{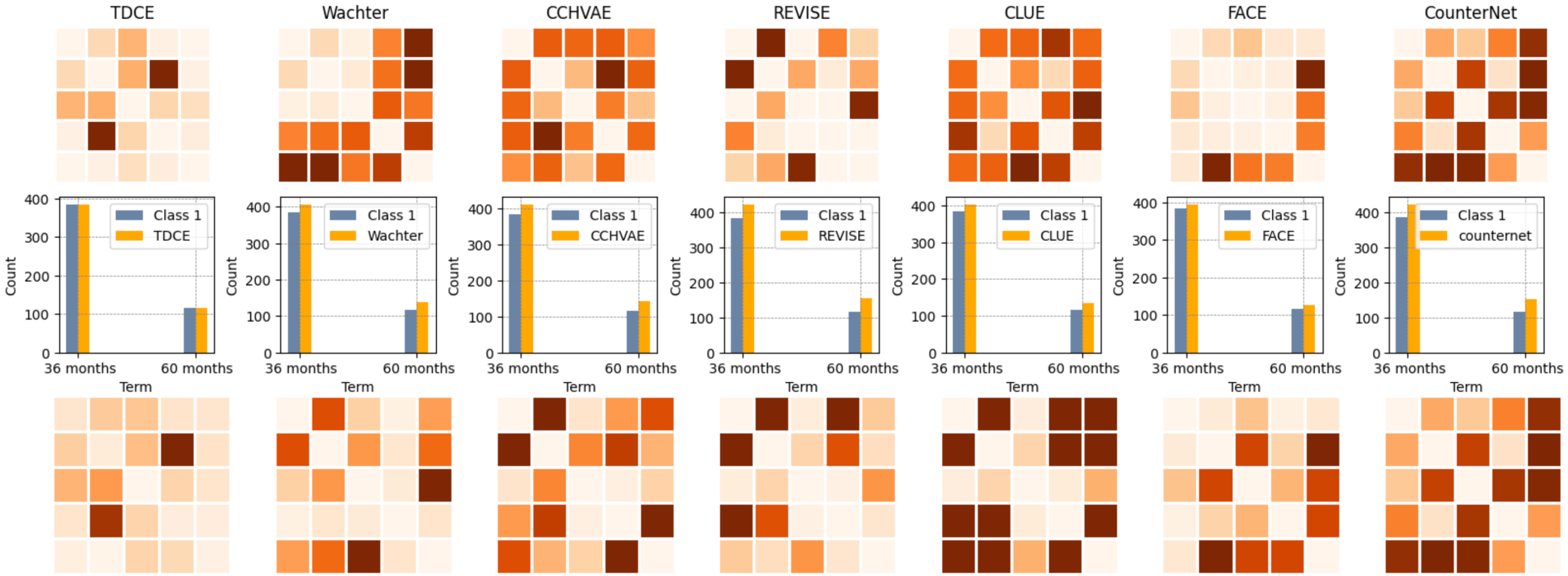}
  \caption{Qualitative comparisons between TDCE and others methods on LCD dataset. Top: the absolute difference between the correlation of counterfactual samples and that of the target class for the continuous features (debt-to-income ratio, loan amount, interest rate, annual income, FICO score). Middle: Bar plots for the categorical variable (loan term: 36 months or 60 months). Bottom: The same metrics as the top when masking the categorical variable. Note: The absolute difference is the same for CounterNet because it simply copies the immutable features from the query sample. }
  \label{fig:qual_plot}
\end{figure*}
\paragraph{\textbf{Qualitative evaluations}} We also provide a qualitative comparison on the LCD dataset in Figure \ref{fig:qual_plot}. LCD contains five numerical features and one categorical feature. In the non-mask setting,  all features are guided by the classifier. Darker pixels show greater discrepancy between the counterfactual sample and the target class. As we can see from the first row, TDCE has fewer darker pixels in general. Importantly, in the middle row, the generated categorical variables from TDCE perfectly match the distribution of the target class. In the masking setting, we fix the categorical variable and only guide the continuous features. In general, the distributional agreement between the target class and the counterfactual class is much greater with TDCE than with other methods.

\paragraph{Discussion on temperature $\tau$}  We evaluate the temperature $\tau$ over IM1, IM2, JS and Validity on LCD dataset in Figure \ref{fig:tau}.The temperature affects the overall counterfactual performance significantly. As the temperature drops, the Gumbel-softmax approaches to a one-hot vector. However, this blocks the flow of gradients from the classifier back to the categorical variable, producing the vanishing gradient issue. In this case, the reverse process is mainly governed by the continuous features. This does not automatically prevent a model from generating valid counterfactual samples because the categorical variables might not be significant in the prediction. However, it does prevent the model from generating realistic counterfactual samples, which diminishes the legitimacy of counterfactual explanations. 
\begin{figure*}[th!]
\centering
  \includegraphics[trim={0 0 0 0cm},clip,width=1\textwidth]{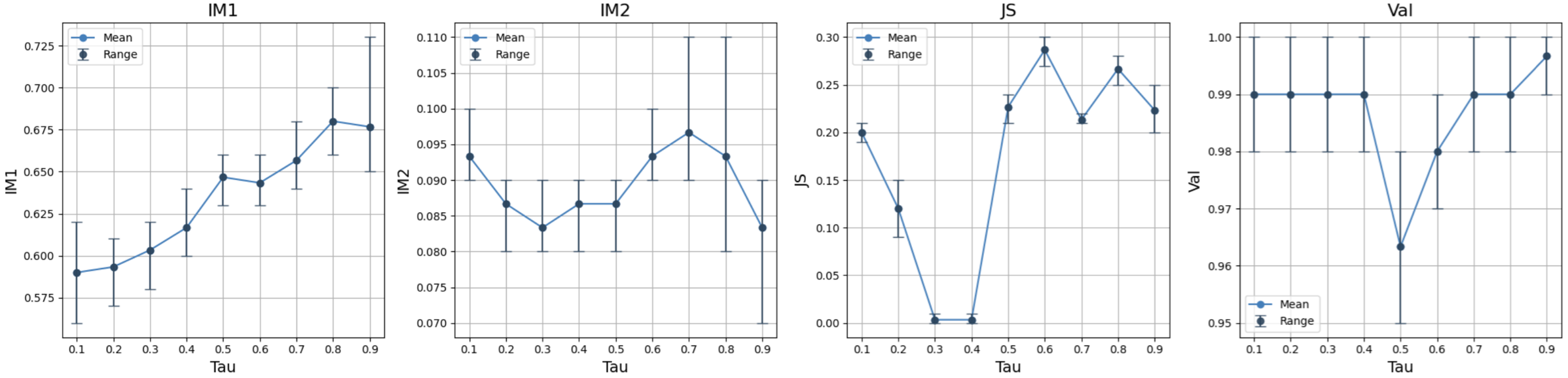}
  \caption{An analysis on the temperature $\tau$ of LCD dataset. The best JS score is achieved when $\tau=0.3$ with balanced IM1 and IM2 scores.}  
  \label{fig:tau}
\end{figure*}
As the temperature increases, the counterfactual generator tends to recover the distribution of the categorical variable in the target class. However, according to Theorem \ref{Thm:thm_kl}, our reverse process might diverge from the true reverse process as the temperature becomes larger. The resulting model then may not be able to recover the distribution of categorical variables well. In the experiments, we search the best $\tau$ ranging from 0.1 to 5 for each dataset. 

\section{Conclusion}
We proposed a tabular diffusion model that generates counterfactual explanations for a classifier. We leverage the Gumbel-softmax distribution to re-parameterize one-hot vectors into a continuous vector, which allows us to utilize the gradients from the classifier to guide the reverse process. We provided theoretical bounds and experimented on four popular tabular datasets. Quantitative results support that our method combines the advantages of working directly in the feature space of Wachter and graph methods with the advantage of neural networks of VAE-based methods.

\bibliographystyle{elsarticle-num-names} 
\bibliography{ref}

\appendix
\newpage
\section{Proof of the closeness of the approximated Gumbel Softmax distribution}\label{Sec:proof}

\begin{lemma}\label{lemma_1}
Given $\pi,\widetilde{x}\in \triangle^{K-1}$ be a probability simplex vector and $\widetilde{x} \geq \widetilde{x}_{min} > 0$. For an arbitrary temperature $\tau\in \mathbb{R}^+$, the lower bound of 
\begin{equation}
    l = \min \sum\limits_i^K \frac{\pi_i}{\widetilde{x}_i^{\tau}}
\end{equation}
has the range
\begin{equation}
    1\leq l \leq K^{\tau}
\end{equation}
The upper bound is:
\begin{equation}
    \sum\limits_i^K \frac{\pi_i}{\widetilde{x}_i^{\tau}} \leq \frac{1}{\widetilde{x}_{min}^{\tau}}
\end{equation}
\end{lemma}
\begin{proof}
We know that $f(u) = u^{\tau}$ for $\tau >0$ is strictly convex. Thus, we know that $\sum\limits_i^K \frac{\pi_i}{\widetilde{x}_i^{\tau}}$ is also strictly convex, indicating that any minimum point is the global minimum. 
Then, we can set up the constrained optimization problem:
\begin{align}
    \min_{x\in \Delta^{K-1}} &F(\widetilde{x})= \sum\limits_i^K \pi_i\widetilde{x}_i^{-\tau} \\
    \text{subject} &  \quad \sum\limits_j^K \widetilde{x}_j = 1
\end{align}
The Lagrangian multiplier can be established as:
\begin{equation}
    \mathcal{L}(\widetilde{x}, \lambda) = \sum\limits_i^K \pi_i\widetilde{x}_i^{-\tau} + \lambda (\sum\limits_j^K \widetilde{x}_j - 1)
\end{equation}
Taking the derivative w.r.t. $\widetilde{x}_j$:
\begin{equation}
    \frac{\partial \mathcal{L}}{\partial \widetilde{x}_j} = \pi_j\widetilde{x}_j^{-\tau-1}(-\tau)+\lambda  = 0 \quad\rightarrow \quad \widetilde{x}_j = (\frac{\tau\pi_j}{\lambda})^{\frac{1}{\tau+1}}
\end{equation}
Enforcing the constraint $\sum_i^K \widetilde{x}_i = 1$ leads to:
\begin{equation}
    \lambda = \tau (\sum\limits_i^K \pi_i^{\frac{1}{\tau+1}})^{\tau+1}
\end{equation}
Combining together:
\begin{equation}
    \widetilde{x}_i^{\ast} = \frac{\pi_i^{\frac{1}{\tau+1}}}{\sum\limits_j^K\pi_j^{\frac{1}{\tau+1}}}
\end{equation}
The minimum value of $F(\widetilde{x})$ is:
\begin{equation}
    F(\widetilde{x}^{\ast}) = (\sum\limits_i^K \pi_i^{\frac{1}{\tau+1}})^{\tau+1}
\end{equation}

By generalized Hölder inequality, this has the range:
\begin{equation}
    1 \leq F(\widetilde{x}^{\ast}) \leq K^{\tau}
\end{equation}
The upper bound can be derived by setting $\widetilde{x}_i = \widetilde{x}_{min}$. Then, the upper bound becomes:
\begin{equation}
    \sum\limits_i^K \frac{\pi_i}{\widetilde{x}_i^{\tau}} \leq \frac{1}{\widetilde{x}_{min}^{\tau}}\sum\limits_{i=1}^K \pi_i = \frac{1}{\widetilde{x}_{min}^{\tau}}
\end{equation}
\end{proof}

    

\newpage
\begin{theorem}
vLet $\widetilde{x}, \pi\in \Delta^{K-1}$ and the temperature $\tau\in \mathbb{R}^+$. Define $\widetilde{x}_{min}$ the minimum value $\widetilde{x}$ can take. The KL divergence between $p_{GS}$ defined in Equation \ref{eqn:gumbel_softmax} and its approximation $p_{\theta}$ in Equation \ref{eqn:approx_gs} is bounded as follows:
\begin{align}
    \mathrm{KL}(p_{GS}\|p_{\theta}) < &  -K(\tau+1)\log \widetilde{x}_{min} + (K-1)\log\tau  + (K-1)\log[1-\widetilde{x}_{min}]  \nonumber\\
    &+\log\Gamma(K) + K\log[(1-\widetilde{x}_{min})/(K-1)!]\nonumber\\
    \mathrm{KL}(p_{GS}\|p_{\theta}) > & ~  K\tau\log\widetilde{x}_{min} + (K-1)\log\tau + (K-1)\log[\widetilde{x}_{min}] + \log\Gamma(K) \nonumber\\
    &+ K\log[\widetilde{x}_{min}/(K-1)!] \nonumber
\end{align}
\end{theorem}

\begin{proof}
    
We want to bound the KL divergence between 
\begin{equation}
    p_{\theta}(\widetilde{x}|\pi) = \frac{1}{Z(\pi)}\prod_i^K \pi^{\widetilde{x}_i}
\end{equation}
and
\begin{equation}
    p(\widetilde{x}|\pi,\tau) = \Gamma(K)\tau^{K-1}\bigg(\sum_i^K\frac{\pi_i}{\widetilde{x}_i^{\tau}}\bigg)^{-K}\prod_i^K \frac{\pi_i}{\widetilde{x}_i^{\tau+1}}
\end{equation}
where $K$ is the number of classes for the categorical variable and $\pi,\widetilde{x}\in \triangle^{K-1}$.

\subsubsection{Derivation of the upper bound}
\begin{align}
    \mathrm{KL}(p(\widetilde{x}|\pi,\tau) \| p_{\theta}(\widetilde{x}|\pi)) & =  \mathbb{E}\Bigg[\log\bigg[\Gamma(K)\tau^{K-1}\prod_{i=1}^K \frac{Z(\pi)\pi_i^{1 - \widetilde{x}_i}\widetilde{x}_i^{-\tau-1}}{\sum\limits_j^K \pi_j \widetilde{x}_j^{-\tau}}\bigg] \Bigg] \\
    &=\mathbb{E}\Bigg[\log\Gamma(K) + (K-1)\log\tau + \sum\limits_i^K\log \frac{Z(\pi)\pi_i^{1 - \widetilde{x}_i}\widetilde{x}_i^{-\tau-1}}{\sum\limits_j^K \pi_j \widetilde{x}_j^{-\tau}}\Bigg] \\
    &\leq \mathbb{E}\Bigg[(K-1)\log\tau + \sum\limits_i^K\log \frac{Z(\pi)(1 - \widetilde{x}_{min})^{1 - \widetilde{x}_i}\widetilde{x}_i^{-\tau-1}}{\sum\limits_j^K \pi_j \widetilde{x}_j^{-\tau}}\Bigg] + \log\Gamma(K)\nonumber \\
    &\text{by} \quad \pi_i^{1 - \widetilde{x}_i} \leq (1 - \widetilde{x}_{min})^{1 - \widetilde{x}_i}\\
    &\leq \mathbb{E}\Bigg[\sum\limits_i^K\log \frac{[\max_{1\leq j \leq K}\pi_j](1 - \widetilde{x}_{min})^{1 - \widetilde{x}_i}\widetilde{x}_i^{-\tau-1}}{(K-1)!\sum\limits_j^K \pi_j \widetilde{x}_j^{-\tau}}\Bigg] + (K-1)\log\tau \nonumber\\
    &+ \log\Gamma(K) \quad \text{Lebesgue measure of } \Delta^{K-1}  \\ 
    & \leq \mathbb{E}\Bigg[\sum\limits_i^K\log \frac{[\max_{1\leq j \leq K}\pi_j](1 - \widetilde{x}_{min})^{1 - \widetilde{x}_i}\widetilde{x}_i^{-\tau-1}}{(K-1)!}\Bigg] + (K-1)\log\tau + \log\Gamma(K) \nonumber \\
    &\quad \text{by Lemma } \ref{lemma_1} \\
    &= \mathbb{E}\Bigg[\sum\limits_i^K\log \widetilde{x}_i^{-\tau-1}\Bigg] + (K+1)\log\tau + \log\Gamma(K) + (K-1)\log(1-\widetilde{x}_{min}) \nonumber \\
    &+ K\log\frac{1-\widetilde{x}_{min}}{(K-1)!}   \\
    &= \mathbb{E}\Bigg[\sum\limits_i^K(-\tau-1)\log \widetilde{x}_i\Bigg] + (K-1)\log\tau + \log\Gamma(K) \nonumber \\
    &+ (K-1)\log(1-\widetilde{x}_{min}) + K\log\frac{1-\widetilde{x}_{min}}{(K-1)!}  \\
    & < -K(\tau+1) \log (1-\widetilde{x}_{min}) + (K-1)\log\tau + \log\Gamma(K) \nonumber \\
    &+  (K-1)\log(1-\widetilde{x}_{min}) + K\log\frac{1-\widetilde{x}_{min}}{(K-1)!}
\end{align}

\subsubsection{Derivation of the lower bound}
\begin{align}
    \mathrm{KL}(p(\widetilde{x}|\pi,\tau) \| p_{\theta}(\widetilde{x}|\pi)) &= \mathbb{E}\Bigg[\log\bigg[\Gamma(K)\tau^{K-1}\prod_{i=1}^K \frac{Z(\pi)\pi_i^{1 - \widetilde{x}_i}\widetilde{x}_i^{-\tau-1}}{\sum\limits_j^K \pi_j \widetilde{x}_j^{-\tau}}\bigg] \Bigg] \\
    &=\mathbb{E}\Bigg[\log\Gamma(K) + (K-1)\log\tau + \sum\limits_i^K\log \frac{Z(\pi)\pi_i^{1 - \widetilde{x}_i}\widetilde{x}_i^{-\tau-1}}{\sum\limits_j^K \pi_j \widetilde{x}_j^{-\tau}}\Bigg] \\
    &\geq \mathbb{E}\Bigg[\log\Gamma(K) + (K-1)\log\tau + \sum\limits_i^K\log \frac{Z(\pi)\pi_i^{1 - \widetilde{x}_i}\widetilde{x}_i^{-\tau-1}}{\widetilde{x}_{min}^{-\tau}}\Bigg] \nonumber \\
    &\text{by Lemma } \ref{lemma_1}  \\
    &= \mathbb{E}\Bigg[(K-1)\log\tau + \sum\limits_i^K\log \frac{Z(\pi)\pi_i^{1 - \widetilde{x}_i}\widetilde{x}_i^{-\tau-1}}{\widetilde{x}_{min}^{-\tau}}\Bigg] + \log\Gamma(K) \\
    &> \mathbb{E}\Bigg[(K-1)\log\tau + \sum\limits_i^K\log \frac{\widetilde{x}_{min}\pi_i^{1-\widetilde{x}_i}\widetilde{x}_i^{-\tau-1}}{(K-1)!\widetilde{x}_{min}^{-\tau}}\Bigg] + \log\Gamma(K) \nonumber \\
    & \text{by Lebesgue measure of } \Delta^{K-1} \\
    &> \mathbb{E}\Bigg[(K-1)\log\tau + \sum\limits_i^K (-\tau-1)(\log\widetilde{x}_i -\log \widetilde{x}_{min}) \Bigg] + \log\Gamma(K) \nonumber \\ 
    &+ K\log\frac{1}{(K-1)!} + (K-1)\log\widetilde{x}_{min} \\ 
    &= \mathbb{E}\Bigg[(K-1)\log\tau + \sum\limits_i^K (\tau+1)(\log \widetilde{x}_{min} -\log\widetilde{x}_i)\Bigg] + \log\Gamma(K) \nonumber \\ 
    &+ K\log\frac{1}{(K-1)!} + (K-1)\log\widetilde{x}_{min} \\ 
    &>  (K-1)\log\tau + K(\tau+1)\log\widetilde{x}_{min} +\log\Gamma(K) \nonumber \\
    &+ K\log\frac{1}{(K-1)!} + (K-1)\log\widetilde{x}_{min} 
\end{align}
\end{proof}

\end{document}